\documentclass[10pt,twocolumn,letterpaper]{article}
\usepackage{fancyhdr}

\usepackage{cvpr}
\usepackage{times}
\usepackage{epsfig}
\usepackage{graphicx}
\usepackage{amsmath}
\usepackage{amssymb}

\usepackage[caption=false,font=footnotesize]{subfig}
\usepackage{xcolor}
\usepackage{booktabs}

\newcommand{\adrian}{\textcolor[rgb]{0,0,0}}
\newcommand{\javier}{\textcolor[rgb]{0,0,0}}

\newcommand\best[1]{\textcolor[rgb]{0.05,0.5,0.02}{\textbf{#1}}}
\newcommand\second[1]{\textcolor[rgb]{1.0,0.5,0.1}{\textbf{#1}}}
\newcommand\third[1]{\textcolor[rgb]{0.25,0.25,0.8}{\textbf{#1}}}

\newcommand*{\qed}{\hfill\ensuremath{\square}}%
\newcommand{\beqn}{\begin{equation*}}
\newcommand{\eeqn}{\end{equation*}}
\newcommand{\beq}{\begin{equation}}
\newcommand{\eeq}{\end{equation}}

\newcommand{\I}{\mathrm{I}}
\newcommand{\J}{\mathrm{J}}

\newcommand{\T}{\mathrm{t}}

\newcommand{\A}{\mathrm{A}}
\newcommand{\G}{\mathrm{G}}
\newtheorem{Th}{Theorem}[section]
\newtheorem{lemma}[Th]{Lemma}

\newenvironment{proof}[1][Proof]{\begin{trivlist}
\item[\hskip \labelsep {\bfseries #1}]}{\end{trivlist}}
\makeatletter
\newcommand*\bigcdot{\mathpalette\bigcdot@{.5}}
\newcommand*\bigcdot@[2]{\mathbin{\vcenter{\hbox{\scalebox{#2}{$\m@th#1\bullet$}}}}}
\makeatother

\usepackage[pagebackref=true,breaklinks=true,letterpaper=true,colorlinks,bookmarks=false]{hyperref}

\cvprfinalcopy 

\def\cvprPaperID{3249} 

\ifcvprfinal\fi

\usepackage[a4paper,margin=0.5in]{geometry}
\usepackage{fancyhdr}
\begin{document}

\title{On the Duality Between Retinex and Image Dehazing}

\author{Adrian Galdran\\
\normalsize{INESC TEC Porto, Portugal}\\
{\tt\small adrian.galdran@inesctec.pt}
\and
Aitor Alvarez-Gila\\
\normalsize{Tecnalia, Spain}\\
\normalsize{Computer Vision Center, Spain}\\
{\tt\small aitor.alvarez@tecnalia.com}
\and
Alessandro Bria\\
\normalsize{Dpt. of Electrical and Information Engineering}\\
\normalsize{University of Cassino and L.M., Italy}\\
{\tt\small a.bria@unicas.it}
\and
Javier Vazquez-Corral, Marcelo Bertalm\'io\\
\normalsize{Dpt. of Information and Communication Technologies}\\
\normalsize{Universitat Pompeu Fabra, Spain}\\
{\tt\small \{javier.vazquez, marcelo.bertalmio\}@upf.es}
}

\maketitle
\thispagestyle{fancy}

\begin{abstract}
Image dehazing deals with the removal of undesired loss of visibility in outdoor images due to the presence of fog. Retinex is a color vision model mimicking the ability of the Human Visual System to robustly discount varying illuminations when observing a scene under different spectral lighting conditions. Retinex has been widely explored in the computer vision literature for image enhancement and other related tasks. While these two problems are apparently unrelated, the goal of this work is to show that they can be connected by a simple linear relationship. Specifically, most Retinex-based algorithms have the characteristic feature of always increasing image brightness, which turns them into ideal candidates for effective image dehazing by directly applying Retinex to a hazy image \javier{whose intensities have been inverted}. In this paper, we give theoretical proof that Retinex on inverted intensities is a solution to the image dehazing problem. Comprehensive qualitative and quantitative results indicate that several classical and modern implementations of Retinex can be transformed into competing image dehazing algorithms performing on pair with more complex fog removal methods, and can overcome some of the main challenges associated with this problem.
\end{abstract}

\section{Introduction}

Outdoor images are often degraded by a loss of visibility produced by small particles lying in the piece of atmosphere in between the imaged scene and the observer. 
This physical phenomenon is known as haze, fog, or mist, and it causes the radiance captured by the camera to be attenuated along its path. 
Haze removal, or image dehazing, is an image processing task concerned with the mitigation of this effect, thereby increasing quality of outdoors images, with the goal of improving performance of further computer vision algorithms, or simply enhancing image visualization.


\javier{In turn, Retinex \cite{land_retinex_1977,land_lightness_1971} was originally defined as a color vision model of human perception. It aims to explain the human ability to perceive color as stable regardless of changes in global illumination. Retinex is based on the observation that color sensation is not related to the radiance values that reach the eye, but to the integrated reflectance. The integrated reflectance is defined as the ratio at each waveband between the value of the object and the value of a white object under the same illuminant. Retinex was promptly adapted by researchers in color photography due to its effectiveness for the enhancement of images \cite{mccann_retinex_2017}. Since then, variations of the Retinex model have been applied for many different image processing tasks, from non-uniform (local) color constancy \cite{ebner_color_2007}, to shadow removal \cite{finlayson_removing_2002}, gamut mapping \cite{mccann_color_2000}, or contrast enhancement \cite{vazquez-corral_image_2016}. In this paper, we consider Retinex as an image enhancement technique, in accordance with these last methods.}

\adrian{Retinex has been related to image dehazing in the past, either explicitly or implicitly. In \cite{xie_improved_2010}, multi-scale Retinex was applied to increase contrast in the luminance channel. The result was then median-filtered and used as an estimate of scene's depth. In \cite{rong_improved_2014}, single-scale Retinex was employed after a wavelet transform to enhance the chromatic aspect of the result, whereas in \cite{dravo_stress_2015} the Stress (Spatio-Temporal Retinex-inspired Envelope with Stochastic Sampling) framework was applied for image dehazing. Stress is a general image enhancement technique, and the authors adapt the behavior of the algorithm to achieve image dehazing through a heuristic adjustment of its parameters.}

\begin{figure*}[tp]
\begin{center}
\includegraphics[width = \textwidth]{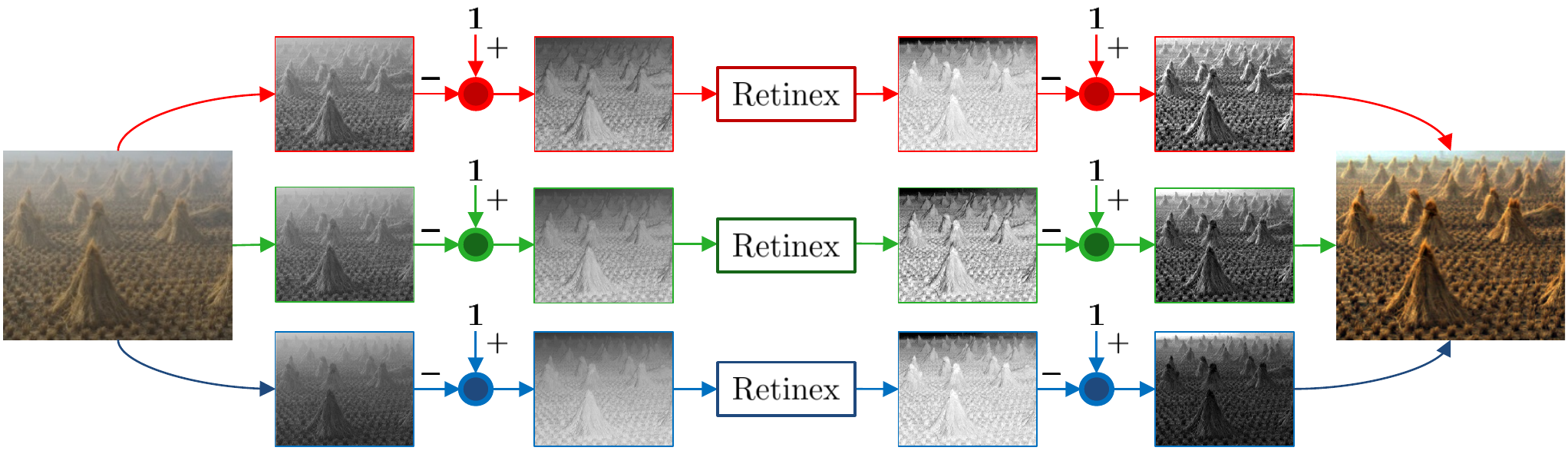}
\end{center}
\caption{A schematic description of the duality relationship between Retinex and Image Dehazing.}
\label{fig_retinex}
\end{figure*}


In contrast with previous works, in this paper we do not intend to adapt Retinex-like ideas to the essentially different problem of image dehazing. Instead, our main contribution is a formal proof of the following direct relationship between Retinex and image dehazing:

\begin{equation}\label{eq_fundamental_1}
\adrian{\mathrm{Dehazing}(\I) = 1-\mathrm{Retinex}(1-\I).}
\end{equation}
Furthermore, we show that this equivalence holds not only at the algorithmic level, but at the modelization level too. This enables the use of existing Retinex-based algorithms to dehazing images directly by incorporating two intensity-inversion operations. This means that we do not need to adjust or modify Retinex-based algorithms to perform image dehazing, we only need to \textit{transform their input} by simple intensity inversion operations. A schematic representation of this process is shown in Fig. \ref{fig_retinex}. In addition, we demonstrate, through a wide set of experimental results, that this new approach to image dehazing can compete surprisingly well with current state-of-the-art fog removal techniques.


\section{Previous Approaches to Image Dehazing}

Many image dehazing techniques have been proposed in recent years. They can be grouped in two main approaches: Machine Learning and Image Processing methods.

Machine Learning techniques learn visual features relevant for classifying an image as hazy or haze-free. 
These features can be manually specified~\cite{choi_referenceless_2015,tang_investigating_2014} or automatically learned in the framework of Deep Convolutional Neural Networks~\cite{cai_dehazenet:_2016,li_aod-net:_2017}. 
A model is then trained to learn a mapping between hazy and haze-free images.
In this case, training examples need to be annotated previously, which is a complex task.
A common approach consists of synthesizing hazy images from natural haze-free images, which is usually accomplished through a physical model of image acquisition under hazy conditions, due to Kochsmieder \cite{koschmieder_theorie_1925}: 
\beq\label{eq_haze_model}
\I(x)=\T(x) \J(x) + (1-\T(x)) \A,
\eeq
where $\I=(\I^R,\I^G,\I^B)$ is the degraded image, $\J$ are the intensities in a haze-free image, $\T$ is the medium transmission, a scalar quantity describing the amount of light that reaches the receiver, inversely related to depth, and $\A$ is a constant (RGB)-vector known as atmospheric light. 
The additive combined degradation of transmission and atmospheric light $\A(1-\T(x))$ is usually known as airlight, and it accounts for a possible shift in scene colors due to the presence of different sources of illumination other than sunlight.

Kochsmieder's model lies also at the heart of image dehazing techniques belonging to the category of Image Processing. 
In this case, the goal is to solve the above underconstrained model (\ref{eq_haze_model}) by building a prior assumption that is fulfilled by a haze-free image. 
This prior is then imposed on eq. (\ref{eq_haze_model}), in order to infer $\T$ and $\A$. Once estimates for $\T$ and $\A$ have been obtained, the eq. (\ref{eq_haze_model}) can be inverted:
\beq\label{eq_haze_inversion}
\J(x) = \frac{\I(x)-\A}{\T(x)}+\A.
\eeq
Image Processing techniques are thus spatially-variant contrast enhancement methods that attempt to increase detail visibility and saturation on degraded areas while leaving unaltered regions that already have good contrast. 
Several priors can be imposed on the structure of $\J$ in order to estimate $\T$ and $\A$. For instance the Dark Channel Prior \cite{he_single_2011} imposes that most local patches in a haze-free image $\J$ contain pixels which have very low intensity in at least one color channel:
\beq\label{dc}
\J^{dark}(x) = \min_{c\in \{R,G,B\}} \left(\min_{y\in\Omega(x)} \J^c(y)\right) \rightarrow 0,
\eeq
being $\Omega(x)$ a local neighborhood of $x$. 
Assuming the Dark Channel Prior is fulfilled by the haze-free image $\J$, we can take minima in eq. (\ref{eq_haze_model}) after normalizing by $\A$, cancel the term associated to $\J$, and recover an estimate of $\T$:
\beq\label{t_computation}
\T(x) = 1 - \min_{c\in \{R,G,B\}} \left(\min_{y\in\Omega(x)}\left(\frac{\I^c(x)}{\A^c}\right)\right).
\eeq
Other haze-free priors can be imposed on $\J$, such as maximal local contrast/saturation \cite{tan_visibility_2008}, or certain distribution of color pixels in the RGB space \cite{berman_non-local_2016,fattal_dehazing_2014}. 
Different alternatives exist: the reader can find in \cite{li_haze_2017,singh_comprehensive_2017} comprehensive reviews.

\adrian{A variation of the above methods consists of dehazing techniques attempting to recover the true physical radiance of the scene objects. These techniques typically require external sources of information \cite{kopf_deep_2008}, or multiple images of the same scene \cite{narasimhan_chromatic_2000,schechner_instant_2001}. Remarkably, in \cite{kratz_factorizing_2009,nishino_bayesian_2012} the authors overcome this need by a joint probabilistic estimation of depth and true radiance through a two-latent-layers Markov random field. The method requires radiometrically calibrated input, and assumes the atmospheric light $\A$ is known in advance, which can result in chromatic distortions \cite{sulami_automatic_2014}.}


%

\begin{figure*}[t]
\centering
\subfloat[]{\includegraphics[width = 0.225\textwidth]{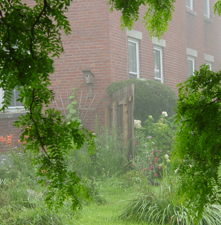}
\label{fig_retinex_increases_1}}
\hfil
\subfloat[]{\includegraphics[width = 0.225\textwidth]{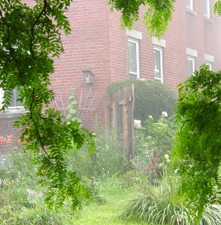}
\label{fig_retinex_increases_2}}
\hfil
\subfloat[]{\includegraphics[width = 0.225\textwidth]{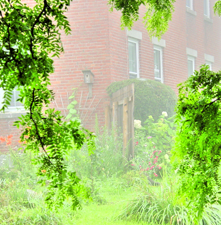}
\label{fig_retinex_increases_3}}
\hfil
\subfloat[]{\includegraphics[width = 0.225\textwidth]{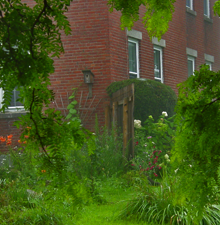}
\label{fig_retinex_increases_4}}
\caption{Retinex cannot decrease brightness. (a) Hazy image (b)-(d) Output of: (b) Random Spray Retinex \cite{provenzi_random_2007} (c) Multi-Scale Retinex \cite{jobson_multiscale_1997} (d) Multi-Scale Retinex used for image dehazing, following eq. (\ref{eq_fundamental_1}).}
\label{fig_retinex_increases}
\end{figure*}

\section{The Retinex Theory of Color Vision}\label{section2}
\javier{Edwin H. Land introduced the Retinex theory \cite{land_retinex_1977} as a color vision model of human perception. He named it Retinex as a portmanteau of Retina and Cortex, since Land did not want to venture where exactly this process was carried out in the visual pathway. In short, the original Retinex color vision model can be defined as \textit{a theoretical spectral channel that makes spatial comparisons between scene regions so as to calculate ``Lightness'' sensations} \cite{mccann_retinex_2017}. It became rapidly apparent to Land and his collaborators that Retinex was also useful for the enhancement of color photographs \cite{mccann_retinex_2017}, replacing the human cone-photoreceptors (L,M,S) by the camera sensors (R,G,B). From now on, we will focus on this second meaning of Retinex, that has been widely applied in image processing tasks \cite{ebner_color_2007,mccann_color_2000,tappen_estimating_2006}.}

\javier{When applied to digital color images, the Retinex model computes a triplet of lightness values $(l_R, l_G, l_B)$ for each pixel.} In the original Retinex implementation, lightness is computed through a chain of pixel intensity comparisons with respect to other image locations' intensities.
Land suggested that this comparison cannot occur directly, but needs to be computed by comparing adjacent pixels  \cite{land_retinex_1977}. Given an image $\I$ taking values in $]0,1]$, two points $x$, $y$, and a path $\gamma=\{y=z_0, z_1, \ldots, z_{n-1}, x=z_n\}$, we compute their ratio $\I(x)/\I(y)$ through consecutive ratios $r_i = \I(z_i)/\I(z_{i-1})$: 
\beq\label{chain}
l^\gamma(x) = \frac{\I(x)}{\I(y)} = \underbrace{\frac{\I(z_1)}{\I(y)}}_{r_1} \cdot\underbrace{\frac{\I(z_2)}{\I(z_1)}}_{r_2} \ldots \underbrace{\frac{\I(z_{n-1})}{\I(z_{n-2})}}_{r_{n-1}} \underbrace{\frac{\I(x)}{\I(z_{n-1})}}_{r_{n}}
\eeq
The unfolding of the $\I(x)/\I(y)$ computation is non-trivial due to the addition of two supplementary mechanisms, called \textit{threshold} and \textit{reset}. The \textit{threshold} mechanism sets to $1$ ratios in eq. (\ref{chain}) that are close to $1$: for a small $\tau$, when $|1-r_i| < \tau$ we set $r_i=1$. This disregards unwanted effects in lightness estimation due to a smooth spatially variant illumination. However, it has been shown that parameter $\tau$ is redundant in Retinex computations \cite{provenzi_mathematical_2005}. Ignoring it does not have a critical impact in the algorithm. Hence, in this work we will not consider threshold-based Retinex variants.

The \textit{reset} mechanism acts as follows: when the chain of computations in (\ref{chain}) reaches a pixel $z_j$ with intensity greater than all previous points in $\gamma$, the sequential product up to $z_j$ resets to $1$, and lightness computation restarts from it:
\begin{align}\label{chain_simplified}
l^\gamma(x) = \frac{\I(x)}{\I(y)} &= \overbrace{r_1 \cdot r_2 \ldots r_{j+1}}^1 \cdot r_j \ldots r_{n-1} \cdot r_n \\ &=  \frac{\I(z_{j+1})}{\I(z_{j})} \cdot \frac{\I(z_{j+2})}{\I(z_{j+1})} \ldots \frac{\I(x)}{\I(z_{n-1})} = \frac{\I(x)}{\I(z_j)}\nonumber
\end{align}
Eq. (\ref{chain_simplified}) shows that the chain of ratios (\ref{chain}) simplifies to $\I(x)/\I(z_{max})$, where $z_{max}$ is the pixel of maximum intensity along $\gamma$. This reveals the local white balance character of Retinex: points activating the reset mechanism become local references for white. 

The sequential product of eq. (\ref{chain}) is scaled by a non-decreasing function $f$, often a logarithm to simplify calculations, and gives an estimate of the lightness in $x$. To improve this estimate, $N$ paths ending at $x$ but starting at different initial points are considered, and the result is averaged, obtaining the Retinex lightness estimate at $x$:
\beq\label{retinex}
l(x) = \frac{1}{N} \sum_{k=1}^N f\left(\frac{\I(x)}{\displaystyle\max_{y\in\gamma^k}\I(y)}\right),
\eeq
where $\gamma^k\in\Gamma=\{\gamma^1,\ldots, \gamma^N\}$, a set of paths on the image domain. Starting from eq. (\ref{retinex}), one can easily show a central property of threshold-free Retinex: it increases brightness. \textit{i.e.} $l(x)\geq\I(x) \ \forall x$ \cite{provenzi_mathematical_2005}. 
This is illustrated in Fig. \ref{fig_retinex_increases}.

In the above form, Retinex contains unspecified parameters, such as the number of paths, or the way in which we sample the image to build them. 
Also, the reset mechanism in eq. (\ref{retinex}) makes much of the paths information redundant. 
In \cite{provenzi_random_2007}, path-based sampling was replaced by sampling through random sprays with radially decreasing density. In \cite{bertalmio_issues_2009}, random sprays wwere replaced by a 2-dimensional representation, with a kernel modeling the sampling density of the spray in the limit, leading to the Kernel-Based Retinex:
\beq\label{kbr}
l(x) = \sum_{y  |  \I(y)\geq\I(x)}\mkern-18mu \omega(x,y) f\left(\frac{\I(x)}{\I(y)}\right) + \sum_{y  |  \I(y)<\I(x)}\mkern-18mu \omega(x,y)
\eeq
where $\omega(x,y)$ models the probability of selecting pixel $y$ in the proximity of $x$, and the reset mechanism is automatically implemented, since $f$ is defined as $f(r)=1$ for $r>1$.

Center-surround techniques were first proposed in \cite{land_alternative_1986} as a simple alternative that still preserves the characteristic features of Retinex. They compute the ratio between image intensity at a pixel and its surrounding:
\beq\label{center_surround}
l(x)  = f\left(\frac{\I(x)}{ <\I(y), y\in \Omega(x)>_w }\right),
\eeq
where $<\cdot>_w$ is a weighted average operator. This amounts to integrating local information instead of sampling it. The first practical implementation of this idea was proposed in \cite{jobson_properties_1997}, where the average operator was a Gaussian kernel $\G_\sigma$:
\beq\label{ssr}
l(x)  = \log\left(\frac{\I(x)}{\G_\sigma*\I(x)}\right) = \log(\I(x)) - \log(\G_\sigma*\I(x)).
\eeq
The scaling function $f$ is here a logarithm. Homomorphic filtering can also be seen as a particular case of this model, in which the logarithm and the convolution occur in inverted order in the right-hand term of (\ref{ssr}). This was later extended to multi-scale Retinex \cite{jobson_multiscale_1997}, a normalized linear combination of (\ref{ssr}) applied with different standard deviations.

Many other flavors of Retinex have been proposed in the literature, e.g. variational \cite{kimmel_variational_2003} or non-local \cite{zhao_closed-form_2012} approaches. We refer to \javier{\cite{mccann_retinex_2017}} for a comprehensive review.

\section{The Duality between Retinex and Image Dehazing}\label{section3}
\adrian{
We begin by observing that any solution to the haze formation model should decrease the intensities of the input hazy image. 
This can be easily seen by rearranging (\ref{eq_haze_model}) into:
\beq\label{geometrical}
\T(x) = \frac{\A-\I(x)}{\A-\J(x)}.
\eeq
Since transmission lies always in $[0,1]$, then $\A-\I(x) \leq \A-\J(x)$, which implies $\J(x) \leq \I(x)$.
}

\adrian{At this point, it is useful to make a simplifying assumption on the haze formation model (\ref{eq_haze_model}). As often done in the image dehazing literature \cite{tarel_fast_2009}, we assume the input image is globally white-balanced,\textit{ i.e.} no chromatic component dominates the scene. This amounts to fixing $\A=(1,1,1)$ in eq. (\ref{eq_haze_model}), and a solution of the image dehazing problem can be rewritten, after a simple manipulation, as:
\beq\label{dehazing_operator}
\mathrm{Dehazing}(\I(x)) = \J(x) = \frac{\I(x)-1}{\T(x)} + 1.
\eeq
}

\begin{figure*}[tp]
\centering
\subfloat[]{\includegraphics[width = 0.256\textwidth]{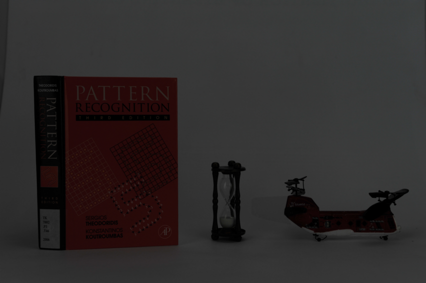}
\label{fig_deh_2}}
\hfil
\subfloat[]{\includegraphics[width = 0.256\textwidth]{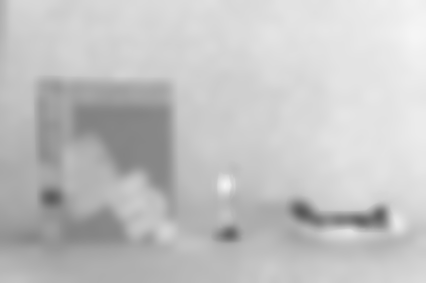}
\label{fig_deh_3}}
\hfil
\subfloat[]{\includegraphics[width = 0.256\textwidth]{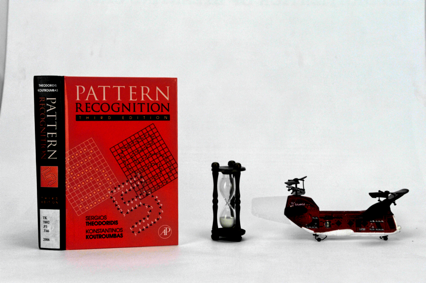}
\label{fig_deh_4}}
\caption{Dehazing as a method for illumination factorization. (a) Original image with irregular illumination. (b) Illumination computed by the dehazing method in \cite{he_single_2011} (scaled for better visualization). (c) Enhanced image with the method of \cite{he_single_2011}. Image from \cite{vonikakis_biologically_2013}.}
\label{fig_deh}
\end{figure*}

\adrian{
In this paper we consider Retinex as an image enhancement technique that can produce, imposing a local color constancy  hypothesis, a uniform illumination image from an image acquired under an irregular illumination:
\beq\label{inv_retinex_0}
\I(x)= \mathrm{Retinex}(\I(x)) \cdot \mathrm{i}(x),
\eeq
where $\mathrm{i}(x)$ is a slowly-varying illumination field affecting the scene. While Retinex produces good results in this ill-posed task, the property of always increasing intensity is a known limitation of most Retinex implementations: they are only able to enhance under-exposed images affected by shadows, while over-exposed images will not be enhanced. 
This limitation is usually circumvented by some further post-processing operations, typically image-dependent and hard to tune.} \adrian{In this work, we turn this limitation into an advantage through the definition of the following operator acting on an image with inverted intensities:
\beq\label{inv_retinex}
\mathrm{DehRet}: \I(x)\rightarrow  1-\mathrm{Retinex}(1-\I(x)).
\eeq
}

\adrian{Note that according to the above observations, if the $\mathrm{DehRet}(\bigcdot)$ solves the Image Dehazing problem, it must share the intensity decreasing property given by eq. (\ref{dehazing_operator}), i.e. the intensities of $\mathrm{DehRet}(\I)$ must be smaller than those of $\I$. This is demonstrated by the following lemma:}
\begin{lemma}
\adrian{
Operator $\mathrm{DehRet}(\bigcdot)$ always decreases intensities.
}
\end{lemma}
\begin{proof}
\adrian{
Since Retinex increases intensities in any image $\I$, for the inverted image $1-\I$ we have $\mathrm{Retinex}(1-\I) \geq 1-\I$. This implies that $1 - \mathrm{Retinex}(1-\I) \leq \I$.\qed
}
\end{proof}

\adrian{
This justifies the suitability of the $\mathrm{DehRet}(\bigcdot)$ operator for the image dehazing task. Now we are ready to prove the central result of this paper.
}
\begin{Th}\label{inverted_retinex_simple}
\adrian{Applying operator $\mathrm{DehRet}(\bigcdot)$ to a hazy image provides a solution of the Image Dehazing problem (\ref{eq_haze_model}).}
\end{Th}
\begin{proof}
Assuming $\A=(1,1,1)$, the haze formation model can be written as:
\beq\label{eq_th1}
\I(x)=\T(x)\J(x)+1-\T(x),
\eeq
where $\J(x)$ is a solution to the dehazing problem, \textit{i.e.} $\J(x) = \mathrm{Dehazing}(\I(x))$. Eq. (\ref{eq_th1}) can be rearranged as:
\beq\label{eq_th2}
1-\I(x) = 1- \T(x) \J(x) - 1 + \T(x) = \T(x)(1-\J(x)).
\eeq
Consider a second image $\widetilde{\I}(x)$ resulting of inverting the intensities of the initial hazy image $\I(x)$, i.e. $\widetilde{\I}(x) = 1-\I(x)$. Eq. (\ref{eq_th2}) can be written as:
\beq\label{eq_th3}
\widetilde{\I}(x) = \T(x)(1-\J(x)).
\eeq
Since $\T(x)$ is piecewise smooth, application of a Retinex method can remove $\T(x)$ from eq. (\ref{eq_th3}), resulting in:
\beq
\mathrm{Retinex}(\widetilde{\I}(x)) = 1-\J(x),
\eeq
which implies:
\beq
\J(x) = 1-\mathrm{Retinex}(\widetilde{\I}(x)).
\eeq
But $\J(x)$ was a solution for the image dehazing problem:
\beq\label{fundamental_result}
\mathrm{Dehazing}(\I(x)) = 1-\mathrm{Retinex}(1-\I(x)),
\eeq
which shows the initial statement.\qed
\end{proof}

\adrian{
The implications of this relationship are manifold. 
First, the above connection between Retinex and Image Dehazing has the advantage that it is valid not only at an algorithmic level, but also at a modelization level. 
It provides a powerful mechanism by which, if we have a numerical technique to solve Retinex, we can solve Dehazing by applying it to inverse intensities and inverting the result.
}

\adrian{
Second, since eq. (\ref{eq_fundamental_1}) holds, a question arises: is it possible to employ dehazing techniques to solve (\ref{inv_retinex_0}), \emph{i.e.} can dehazing on inverted intensities remove a smooth illumination field from an irregularly illuminated image?
In the considered case of a neutral-color illumination, through a change of variables $\I\rightarrow 1-\I$, eq. (\ref{eq_fundamental_1}) can be re-written as:
\begin{equation}\label{eq_fundamental_2}
1-\mathrm{Dehazing}(1-\I) = \mathrm{Retinex}(\I).
\end{equation}
This implies that inverting the result of running a dehazing method on inverted intensities will return an illumination-free image.
In addition, it can be easily shown that the operator $\I \rightarrow 1-\mathrm{Dehazing}(1-\I)$ is non-decreasing. This means that this operator can be applied to remove illumination, although it will only work for under-exposed images.
}

\adrian{
Indeed, eqs. (\ref{fundamental_result}) and (\ref{eq_fundamental_2}) build up a bidirectional image processing tool. 
Not only can algorithms for illumination factorization be applied to remove fog, but also image dehazing techniques can factor out non-uniform illumination from under-exposed images, as Retinex does. 
An example of the effect of applying formula (\ref{eq_fundamental_2}), with the dehazing method from \cite{he_single_2011}, is shown in Fig. \ref{fig_deh}. This idea has been recently explored in several works related to low-light image enhancement \cite{panagopoulos_estimating_2010,cai_joint_2017,li_low-light_2015,guo_lime:_2017}. 
The above result can be regarded as providing a theoretical support to these works.
}


\begin{figure*}[t]
\begin{center}
\subfloat[]{\includegraphics[width = 0.24\textwidth]{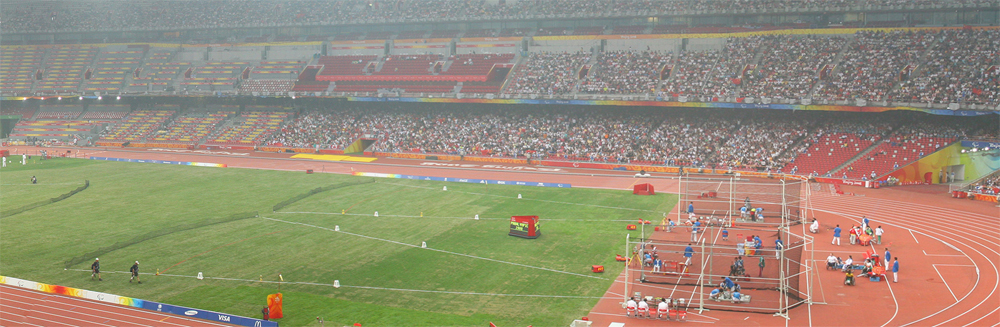}
\label{fig_retinex_increases_1}}
\hfil
\subfloat[]{\includegraphics[width = 0.24\textwidth]{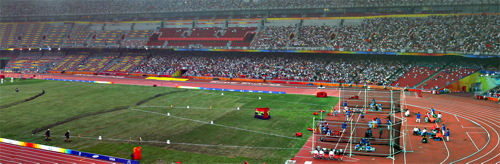}
\label{fig_retinex_increases_2}}
\hfil
\subfloat[]{\includegraphics[width = 0.24\textwidth]{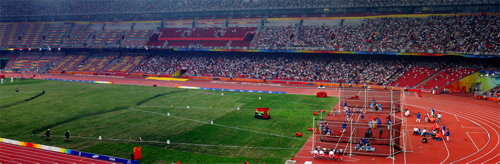}
\label{fig_retinex_increases_3}}
\hfil
\subfloat[]{\includegraphics[width = 0.24\textwidth]{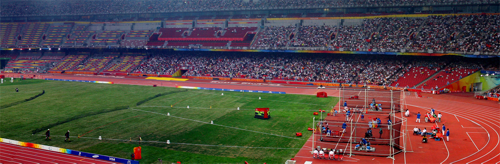}
\label{fig_retinex_increases_4}}

\subfloat[]{\includegraphics[width = 0.24\textwidth]{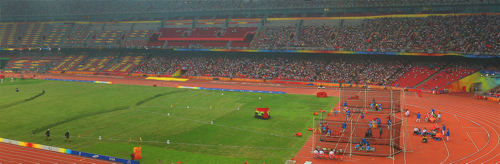}
\label{fig_retinex_increases_5}}
\hfil
\subfloat[]{\includegraphics[width = 0.24\textwidth]{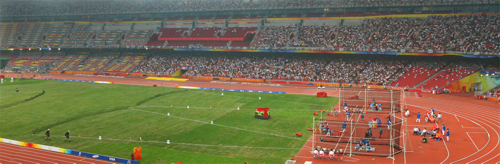}
\label{fig_retinex_increases_6}}
\hfil
\subfloat[]{\includegraphics[width = 0.24\textwidth]{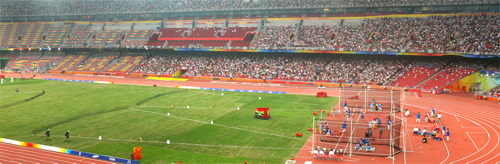}
\label{fig_retinex_increases_1}}
\hfil
\subfloat[]{\includegraphics[width = 0.24\textwidth]{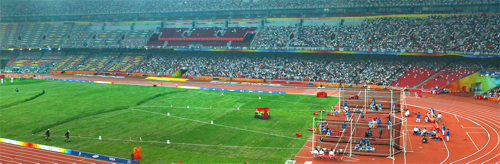}
\label{fig_retinex_increases_2}}

\subfloat[]{\includegraphics[width = 0.24\textwidth]{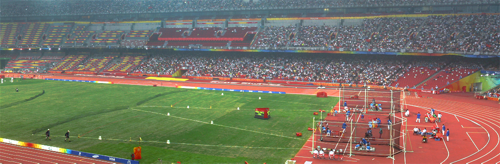}
\label{fig_retinex_increases_3}}
\hfil
\subfloat[]{\includegraphics[width = 0.24\textwidth]{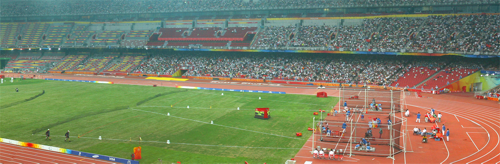}
\label{fig_retinex_increases_4}}
\hfil
\subfloat[]{\includegraphics[width = 0.24\textwidth]{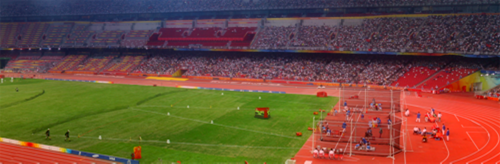}
\label{fig_retinex_increases_5}}
\hfil
\subfloat[]{\includegraphics[width = 0.24\textwidth]{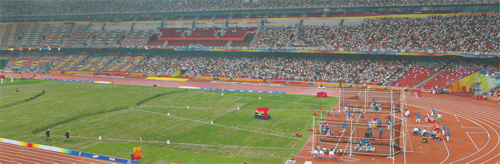}
\label{fig_retinex_increases_6}}
\end{center}
\caption{Results of Dehazing by means of Retinex on inverted intensities (marked with asterisk and boldface) as compared to several popular Image Dehazing techniques. (a) Original (b) \textbf{MSR$^\ast$}\cite{jobson_multiscale_1997} (c) \textbf{RSR$^\ast$}\cite{provenzi_random_2007} (d) \textbf{LRSR$^\ast$}\cite{banic_light_2013} (e) \textbf{HF$^\ast$} (f) \textbf{WVRI$^\ast$} \cite{fu_weighted_2016} (g) DCP \cite{he_single_2011} (h) BCCR\cite{meng_efficient_2013} (i) DEFADE \cite{choi_referenceless_2015} (j) CAP\cite{zhu_fast_2015} (k) RAS\cite{chen_robust_2016} (l) FVR\cite{tarel_fast_2009}.}
\label{fig_stadium}
\end{figure*}

\subsection{Dark Channel as Retinex}
\adrian{
Consider a monochromatic hazy image acquired under neutral illumination. From eq. (\ref{t_computation}), transmission reduces to:
\beq\label{t_computation_simple}
\T(x) = 1 - \min_{y\in\Omega(x)}(\I(x)).
\eeq
Consider $r\in [0,1]$, and a neighborhood around it given by $V(r)\subset[0,1]$. The following property holds:
\beq
\displaystyle\min_{s\in V(r)} \ s = 1-\max_{s\in V(r)}(1-s)
\eeq
Thus, the transmission that the Dark Channel computes from an image after inverting its intensities is given by:
\beq
\T_{1-\I}(x) = 1 - \min_{y\in\Omega(x)}(1-\I(y)) = \max_{y\in\Omega(x)}(\I(y)).
\eeq
It becomes apparent now that the solution the Dark Channel computes for the Retinex problem relates the denominator of the  haze inversion formula (\ref{eq_haze_inversion}) to the denominator of the Retinex equation (\ref{retinex}). 
However, in this case the scaling function $f$ is the identity and the geometry of the neighborhoods is the simplest one: square neighborhoods with no weighting factor. Hence the need for refining $\T$ that affects this algorithm, as well as other techniques derived from it.
}

\section{Experiments and Results}
The connections demonstrated in section \ref{section3} are not tied to one specific Retinex method, but they hold at a fundamental level. 
Hence, in order to verify the validity of the proposed image dehazing approach, we only require that the applied algorithm is able to separate a smoothly variant illumination field from the reflectance of the scene, in a way consistent with the assumptions outlined in section \ref{section3}. For this reason, we have selected four different popular implementations of Retinex: Single Scale Retinex (\textbf{SSR})  \cite{jobson_properties_1997}, Multi-Scale Retinex (\textbf{MSR}) \cite{jobson_multiscale_1997}, Random Spray Retinex (\textbf{RSR}) \cite{provenzi_random_2007} and its faster version, Light Random Spray Retinex (\textbf{LRSR}) \cite{banic_light_2013}. In addition, we include the Homomorphic Filtering (\textbf{HF}), which can be interpreted as a member of the Retinex family, as well as a recent illumination-reflectance separation technique (\textbf{WVRI}) \cite{fu_weighted_2016}.

These techniques are executed on inverted intensities, and inverted afterwards, following Theorem \ref{inverted_retinex_simple}. We deliberately prefer not to perform extensive parameter optimization over Retinex implementations, so as to show their general behavior for the Image Dehazing task. 
Since MSR and SSR operate on the logarithmic domain, we map back their results to $[0,1]$ by simple affine translation, saturating a small percentage of pixels at both extremes ($1\%$). This operation is applied on the result of $\mathrm{Retinex(1-\I)}$, maintaining the property of not decreasing intensity values of Retinex. The remaining parameters are fixed as the default values proposed by the respective authors ($\sigma_{L}=15$, $\sigma_{M}=80$, $\sigma_{H}=250$ for MSR and $\sigma=80$ for SSR). The spray size for both $RSR$ and $LRSR$ is set to $n=75$, with $N=20$ and $N=1$ number of sprays, respectively. For LRSR, kernel sizes are $k1=k2=25$, and row and column step sizes are both $1$.
The remaining methods were also executed with the baseline parameter configuration provided by their respective authors.

Below we compare both qualitatively and quantitatively the result of our proposed approach with a wide set of well-established image dehazing techniques: the popular Dark Channel Prior (\textbf{DCP}) \cite{he_single_2011}, the Fast Visibility Restoration (\textbf{FVR}) technique of \cite{tarel_fast_2009}, Image Dehazing with Robust Artifact Suppression (\textbf{RAS}) \cite{chen_robust_2016}, \textbf{DEFADE} \cite{choi_referenceless_2015}, Bayesian Defogging (\textbf{BYD}) \cite{nishino_bayesian_2012}, the Boundary-Constrained Contextual Regularization technique (\textbf{BCCR}) \cite{meng_efficient_2013}, \textbf{EVID} \cite{galdran_enhanced_2015}, \textbf{FVID} \cite{galdran_fusion-based_2017}, and the Color Attenuation Prior (\textbf{CAP}) technique \cite{zhu_fast_2015}. We also consider Histogram Equalization, to analyze the comparative performance of a simple contrast enhancement method. We must stress that our goal is not to produce results largely improving those of the Image Dehazing state-of-the-art, but to demonstrate the general usability of existing Retinex implementations for the task of fog removal.

\subsection{Qualitative Evaluation}

\begin{figure*}[t]
\centering
\subfloat[]{\includegraphics[width = 0.155\textwidth]{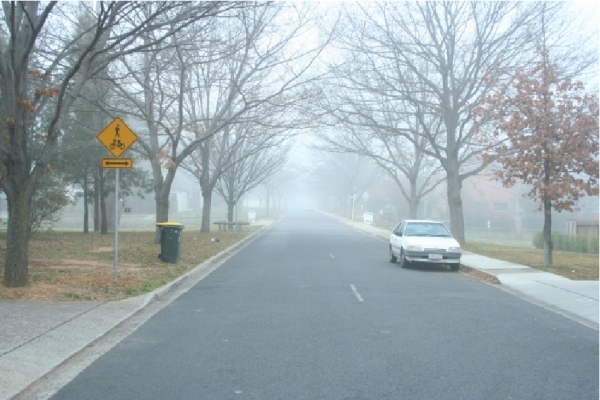}
\label{fig_retinex_increases_1}}
\hfil
\subfloat[]{\includegraphics[width = 0.155\textwidth]{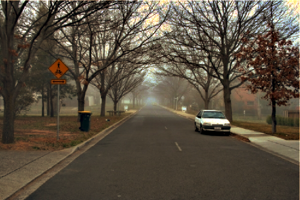}
\label{fig_retinex_increases_2}}
\hfil
\subfloat[]{\includegraphics[width = 0.155\textwidth]{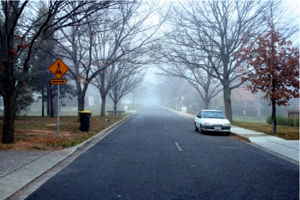}
\label{fig_retinex_increases_3}}
\hfil
\subfloat[]{\includegraphics[width = 0.155\textwidth]{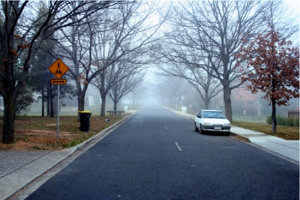}
\label{fig_retinex_increases_4}}
\hfil
\subfloat[]{\includegraphics[width = 0.155\textwidth]{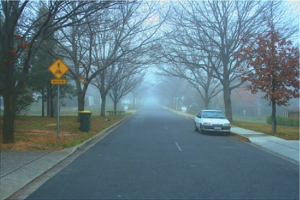}
\label{fig_retinex_increases_5}}
\hfil
\subfloat[]{\includegraphics[width = 0.155\textwidth]{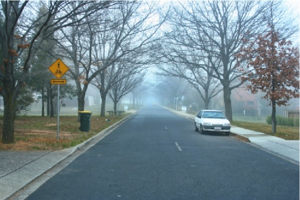}
\label{fig_retinex_increases_6}}

\subfloat[]{\includegraphics[width = 0.155\textwidth]{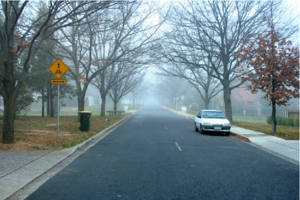}
\label{fig_retinex_increases_1}}
\hfil
\subfloat[]{\includegraphics[width = 0.155\textwidth]{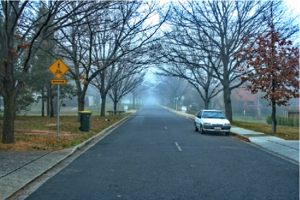}
\label{fig_retinex_increases_2}}
\hfil
\subfloat[]{\includegraphics[width = 0.155\textwidth]{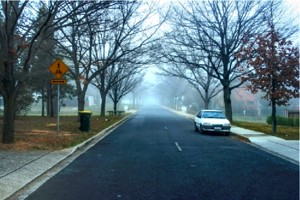}
\label{fig_retinex_increases_3}}
\hfil
\subfloat[]{\includegraphics[width = 0.155\textwidth]{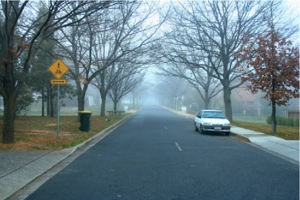}
\label{fig_retinex_increases_4}}
\hfil
\subfloat[]{\includegraphics[width = 0.155\textwidth]{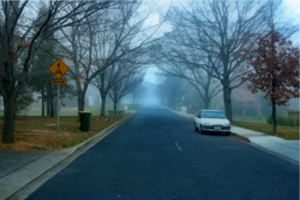}
\label{fig_retinex_increases_5}}
\hfil
\subfloat[]{\includegraphics[width = 0.155\textwidth]{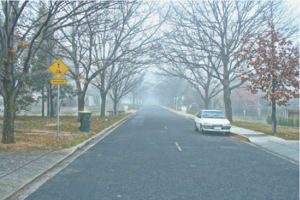}
\label{fig_retinex_increases_6}}
\caption{Results of Dehazing by means of Retinex on inverted intensities (marked with asterisk and boldface) as compared to several popular Image Dehazing techniques. (a) Original (b) \textbf{MSR$^\ast$}\cite{jobson_multiscale_1997} (c) \textbf{RSR$^\ast$}\cite{provenzi_random_2007} (d) \textbf{LRSR$^\ast$}\cite{banic_light_2013} (e) \textbf{HF$^\ast$} (f) \textbf{WVRI$^\ast$} \cite{fu_weighted_2016} (g) DCP \cite{he_single_2011} (h) BCCR\cite{meng_efficient_2013} (i) DEFADE \cite{choi_referenceless_2015} (j) CAP\cite{zhu_fast_2015} (k) RAS\cite{chen_robust_2016} (l) FVR\cite{tarel_fast_2009}.}
\label{fig_road}
\end{figure*}

In this section, we show several visual examples of the application of operator $\mathrm{DehRet}(\bigcdot)$ with the Retinex algorithms mentioned above, compared to the result of applying Image Dehazing techniques\footnote{Further qualitative results can be found in the supplementary material}. Fig. (\ref{fig_stadium}) displays a first example of such results. As predicted by Theorem (\ref{inverted_retinex_simple}), Retinex-based techniques can improve visibility to an extent similar to that of other specialized fog removal algorithms, showing good contrast and saturation on areas that are far away from the camera. Even the simple Homomorphic Filtering has a good performance in retrieving visibility in those areas.

Figure (\ref{fig_road}) provides an interesting example. Application of formula (\ref{inv_retinex}) leads again to a visibility increase on areas of the scene's bottom. In this case, the result of different implementations of Retinex produces colors that are sometimes unnatural. This is related to the per-channel processing Retinex performs. The role of the atmospheric light in eq. (\ref{eq_haze_model}) is ignored in this implementation, leading to a disparate color recovery in different images. Although the performance of Retinex under the presence of color shifts is reasonable, the relationship between Retinex and Image Dehazing when the term $\A$ is considered is complex, and remains a topic of future research.


\subsection{Quantitative Evaluation}
There exist two different approaches to quantitatively assess the quality of an image dehazing method, namely by full-reference metrics, and by no-reference metrics. 
In the first case, a ground-truth optimal solution is assumed to exist, and the error between the result of a dehazing technique and its corresponding clean scene can be computed.  
In the second case, a score describing the quality of a hazy image and its dehazed counterpart can be analyzed without the need of a clear version of the original image.
Below we follow both approaches to verify the applicability of the $\mathrm{DehRet}(\bigcdot)$ operator as defined on eq. (\ref{inv_retinex}) to increase the visual quality of images degraded by haze.

\subsubsection{Full-Reference Quality Assessment}
We first assess the performance of Retinex-based techniques for the dehazing problems by means of full-reference metrics. 
We use a set of outdoors images on which synthetic fog is added through perturbed versions of the haze formation model of eq. (\ref{eq_haze_model}), following \cite{galdran_enhanced_2015}. 
In this dataset it is possible to compute full-reference error measurements. 
Table \ref{ssim_table} shows the obtained results after applying all considered image dehazing and Retinex-based techniques, and measuring deviation with respect to the haze-free groundtruth image in terms of the well-known Structural Similarity Index (SSIM) \cite{wang_image_2004}, Color Peak Signal-to-Noise Ratio (CPSNR), and $\Delta E_{00}$ \cite{sharma_ciede2000_2005} mean errors across the dataset. Numerical results confirm that the proposed approach shows a dehazing capability in line with that of current fog removal methods, sometimes even outperforming it. Overall, the best-performing techniques were the DCP \cite{he_single_2011} and the weighted variational method for illumination separation from \cite{fu_weighted_2016}, acting on inverted intensities. These methods achieved a first and a second place under two different metrics. First, second, and third best performing methods were relatively well-distributed between image dehazing and Retinex-based techniques, which supports the hypothesis that Retinex methods can compete with specialized fog removal algorithms.

\begin{table*}[t]
\centering
\begin{tabular}{lcccccccc}
\textbf{Method}&                BYD \cite{nishino_bayesian_2012}&  HE&  DCP \cite{he_single_2011} &  EVID \cite{galdran_enhanced_2015}&  FVID \cite{galdran_fusion-based_2017}&  FVR\cite{tarel_fast_2009}&  BCCR\cite{meng_efficient_2013}&  DEFADE \cite{choi_referenceless_2015}  \\[1mm]\hline\\[-2.5mm]
\textbf{SSIM}    & 0.489 & 0.671 & \best{0.808} & 0.763 & 0.781 & 0.775 & \second{0.792} & 0.0.716\\[1.5mm]
\textbf{CPSNR}   & 11.569 & 12.368 & 15.479 & 15.086 & 14.695 & 15.261 & 16.085  & 14.210 \\[1.5mm]
\textbf{$\Delta E_{00}$}    &  20.312 & 18.459 & \second{10.499} & 13.422 & 13.530 & 12.387 & 11.239 & 14.098 \\[1.5mm]\hline\hline\\[-2.5mm]
\textbf{Method} \hspace{0.5cm}  & CAP\cite{zhu_fast_2015}&  RAS\cite{chen_robust_2016}&  \textbf{HF$^\ast$}&  \textbf{WVRI$^\ast$} \cite{fu_weighted_2016}&  \textbf{SSR$^\ast$}\cite{jobson_properties_1997}&  \textbf{MSR$^\ast$}\cite{jobson_multiscale_1997}&  \textbf{RSR$^\ast$}\cite{provenzi_random_2007}&  \textbf{LRSR$^\ast$}\cite{banic_light_2013}  \\[1mm]\hline\\[-2.5mm]
\textbf{SSIM}    & 0.709 & 0.507 & 0.782 & 0.702 & 0.733 & 0.742 & 0.739 & \third{0.788} \\[1.5mm]
\textbf{CPSNR}     & \third{16.160} & 14.483 & \best{16.526} & \second{16.388} & 13.887 & 14.497 & 15.674 & 15.741 \\[1.5mm]\textbf{$\Delta E_{00}$}   & \third{10.501} & 14.392 & 10.586 & \best{10.425} & 16.346 & 15.162 & 12.135 & 12.007
\\[1.5mm]\bottomrule\\[0.1mm]
\end{tabular}
\caption{SSIM/CPSNR/$\Delta E_{00}$ errors for synthetic foggy images from \cite{galdran_enhanced_2015}. For each metric, best method is marked \best{green}, second best is marked \second{orange}, and third best is marked \third{blue}. Methods based on our \emph{Retinex for Dehazing} approach are marked bold with a $^\ast$ sign.}\label{ssim_table}
\end{table*}

\subsubsection{No-Reference Quality Assessment}
\begin{table*}[t]
\centering
\begin{tabular}{lcccccccc}
\textbf{Method}&         None&  HE&  DCP \cite{he_single_2011} &  EVID \cite{galdran_enhanced_2015}&  FVID \cite{galdran_fusion-based_2017}&  FVR\cite{tarel_fast_2009}&  BCCR\cite{meng_efficient_2013}&  DEFADE \cite{choi_referenceless_2015}  \\[1mm]\hline\\[-2.5mm]
\textbf{FADE-score}    & 1.556 & 1.125 & 0.870 & 0.691 & 0.930 & 0.748 & \second{0.564} & \best{0.517}\\[1.5mm]
\textbf{$\mathbf{e}$-score}    & - & 1.477 & 0.953 & 1.481 & 0.753 & 1.181 & 1.612  & 0.923 \\[1.5mm]
\textbf{$\mathbf{r}$-score}    & -    &1.853 & 1.1513 & 1.855 & 1.307 & 1.931 & \third{1.977}  & 1.434 \\[1.5mm]
\textbf{$\mathbf{\sigma}$-score}    &  - & 1.019 & 0.095 & \second{0.018} & 0.072 & 0.150 & 0.353 & 5.727 \\[1.5mm]\hline\hline\\[-2.5mm]
\textbf{Method} \hspace{0.5cm}  & CAP\cite{zhu_fast_2015}&  RAS\cite{chen_robust_2016}&  \textbf{HF$^\ast$}&  \textbf{WVRI$^\ast$} \cite{fu_weighted_2016}&  \textbf{SSR$^\ast$}\cite{jobson_properties_1997}&  \textbf{MSR$^\ast$}\cite{jobson_multiscale_1997}&  \textbf{RSR$^\ast$}\cite{provenzi_random_2007}&  \textbf{LRSR$^\ast$}\cite{banic_light_2013}  \\[1mm]\hline\\[-2.5mm]
\textbf{FADE-score}    & 1.048 & 0.625 & 0.941 & 0.823 & 0.644 & \third{0.575} & 0.677 & 0.654 \\[1.5mm]
\textbf{$\mathbf{e}$-score}     & 0.241 & 0.1044 & 1.207 & 0.292 & \third{1.861} & \second{2.056} & \best{2.421} & 1.341 \\[1.5mm]\textbf{$\mathbf{r}$-score}    & 1.033 & 0.891 & 1.146 & 1.072 & \best{2.137} & \second{2.077} & 1.702 & 1.285 \\[1.5mm]
\textbf{$\mathbf{\sigma}$-score}    & 0.173 & 4.047 & 0.186 & 0.144 & 0.686 & 0.683 & \best{0.002} & \third{0.066}
\\[1.5mm]\bottomrule\\[0.1mm]
\end{tabular}
\caption{Quantitative results on FADE, $e$, $r$, and $\sigma$ metrics. For each metric, best method is marked \best{green}, second best is marked \second{orange}, and third best is marked \third{blue}. Methods based on our \emph{Retinex for Dehazing} approach are marked bold with a $^\ast$ sign.}\label{tab_results}
\end{table*}

For a no-reference assessment, we evaluate the proposed Retinex-based approach by conducting a series of experiments on the dataset provided in \cite{choi_referenceless_2015}, which is publicly available online\footnote{\url{http://live.ece.utexas.edu/research/fog/index.html}}.
This dataset comprises $500$ natural hazy images of varying sizes, fog density and content, and includes most of the typical test images used in most previous works.

We now compare Retinex-based implementations with results obtained on the same dataset by the set of state-of-the-art image dehazing algorithms from the previous section.
The Perceptual Fog Density measure (FADE) proposed in \cite{choi_referenceless_2015} is employed. We also consider three extra quality metrics, introduced in \cite{hautiere_blind_2011}: $e$, $r$, and $\sigma$, reflecting different aspects of the quality of dehazed images, \emph{i.e.} percentage of new visible edges after the enhancement process ($e$), increase of visibility/contrast level ($r$), and percentage of pixels becoming saturated after processing an image ($\sigma$).


We report in Table~\ref{tab_results} the mean of the FADE metric and the $e$, $r$, $\sigma$ coefficients for the aforementioned set of $500$ images. Several interesting conclusions can be drawn. First, notice that in terms of the FADE score, the best-performing technique is DEFADE. However, this is a machine learning approach that was trained to remove fog on the same image set we analyze here. Thus, its good performance is expected. As for the FADE score, Retinex-based methods seem to perform on pair with image dehazing techniques, which verifies the duality proposed in this paper. This is confirmed by the $e$, $r$, $\sigma$ scores, which point to the RSR technique as capable of revealing new visible edges while avoiding to saturate previously unsaturated pixels. Finally, we notice that Histogram Equalization performs poorly when compared to other techniques, confirming that the task of fog removal is substantially different from simple spatially-invariant contrast increasing, and that Retinex on inverted intensities can fulfill that task successfully.

\section{Conclusions}
In this work we have provided a rigorous mathematical proof of the dual relationship connecting the problems of image dehazing and non-uniform illumination separation, showing that applying a Retinex operation on an inverted image followed by inverting the result again provides a dehazed result, and vice versa. Rather than being limited to a particular algorithm, we have formally and experimentally showed that this holds for a wide range of Retinex methods. Qualitative and quantitative experiments showed competitive results when compared to current dehazing algorithms.

\section*{Acknowledgments}
JVC was supported by the Spanish government grant ref. IJCI-2014-19516, and MB by European Research Council, Starting Grant ref. 306337, by the Spanish government grant ref. TIN2015-71537-P, \& by Icrea Academia Award.

{\small
\bibliographystyle{ieee}
\bibliography{dehazing_cvpr_refs}
}

\end{document}